\documentclass[letterpaper]{IEEEtran}
\usepackage{latexsym,,amssymb,amsmath,graphicx,epsf,cite,bbm,float}
\usepackage{ifpdf}
\usepackage{epstopdf}

\usepackage{algorithm,algorithmic}
\usepackage{amsmath,amssymb,bm}
\usepackage{amsfonts,dsfont,color,bbm,subcaption}

\usepackage[linguistics]{forest}

\def\ninept{\def\baselinestretch{1}}
\ninept

\newcommand{\be}{\begin{equation}}
\newcommand{\ee}{\end{equation}}
\newcommand{\bea}{\begin{eqnarray}}
\newcommand{\eea}{\end{eqnarray}}

\newcommand{\MB}{\left[\begin{array}}
\newcommand{\ME}{\end{array}\right]}

\newcommand{\ei}{\end{itemize}}
\newcommand{\bi}{\begin{itemize}}

\newcommand{\X}{\mathcal{X}}

\newcommand{\Z}{\mathcal{Z}}

\newcommand{\A}{\mathcal{A}}
\newcommand{\B}{\mathcal{B}}
\newcommand{\C}{\mathcal{C}}
\newcommand{\D}{\mathcal{D}}

\usepackage{amsmath,amsthm}
\usepackage{hyperref}

\DeclareMathOperator*{\argmin}{arg\,min}

\newtheorem{theorem}{Theorem}

\newtheorem{lemma}[]{Lemma}

\newtheorem{definition}[]{Definition}

\newtheorem{assumption}[]{Assumption}

\newtheorem{proposition}[]{Proposition}

\begin{document}

\title{Sequential Linearithmic Time Optimal Unimodal Fitting When Minimizing Univariate Linear Losses} 
\author{
	\IEEEauthorblockN{Kaan Gokcesu}, \IEEEauthorblockN{Hakan Gokcesu}
}
\maketitle

\begin{abstract}
	This paper focuses on optimal unimodal transformation of the score outputs of a univariate learning model under linear loss functions. We demonstrate that the optimal mapping between score values and the target region is a rectangular function. To produce this optimal rectangular fit for the observed samples, we propose a sequential approach that can its estimation with each incoming new sample. Our approach has logarithmic time complexity per iteration and is optimally efficient.
\end{abstract}

\section{Introduction}

In detection, estimation, prediction and learning problems \cite{poor_book, cesa_book}; intelligent agents often make decisions while dealing with considerable uncertainty such as randomness, noise and incomplete data. In such cases, the agents combine various features to determine the actions that maximize some utility \cite{russel2010}. These types of problems are prevalent in several fields, including decision theory \cite{tnnls4}, control theory \cite{tnnls3}, game theory \cite{tnnls1, chang}, optimization \cite{gokcesu2022low,zinkevich, hazan}, distribution estimation \cite{gokcesu2018density, willems, gokcesu2018anomaly, coding2}, anomaly detection \cite{gokcesu2019outlier}, signal processing \cite{ozkan}, prediction \cite{singer,gokcesu2020recursive} and bandits \cite{cesa-bianchi}.
The output of these learning models aims to discriminate the data patterns and provide accurate estimates for practical usefulness. While most learning methods produce output scores, many applications require accurate target estimates. Therefore, it has become crucial to develop methods for post-processing the output of learning models to generate accurate estimates. 

One example is the calibration of classifiers \cite{naeini2015obtaining}. Instead of a produced score, the end application requires a probability estimate. Producing well-calibrated probabilities is critical in many fields, including science (e.g., determining which experiment to conduct), medicine (e.g., choosing which therapy to use) and business (e.g., deciding which investment to make). In learning problems, obtaining well-calibrated classifiers is essential not only for decision-making but also for combining \cite{bella2013effect} or comparing \cite{zhang2004naive,jiang2005learning, hashemi2010application} different classifiers. In classification problems probability estimates play a vital role \cite{zadrozny2001obtaining}. There are two approaches to obtaining well-calibrated classification models. The first approach is to build an intrinsically well-calibrated model by modifying the objective function, which can potentially increase the computational cost \cite{naeini2015obtaining}. On the other hand, the second approach involves post-processing the outputs of discriminative classification models to achieve calibration, which is flexible and general \cite{naeini2015obtaining}. However, it can potentially decrease discrimination while increasing calibration if not done carefully.
Calibration methods have two main applications: they can be used to convert the outputs of discriminative models that have no apparent probabilistic interpretation to posterior class probabilities, or to improve the calibration of a mis-calibrated model \cite{platt1999probabilistic, niculescu2005predicting}. 

In literature, there are different techniques for mapping the outputs of learning models to accurate estimates.
One way is by imposing a parametric form such as Platt's method \cite{platt1999probabilistic}, which uses a sigmoid function to map the score outputs into probabilities and then maximizes the likelihood of the parameters using a model-trust minimization algorithm \cite{gill2019practical}. The method was initially developed for SVM models, but it has also been applied to other classifiers, including Naive Bayes \cite{niculescu2005predicting}. The sigmoid function, however, may not be the best fit for Naive Bayes scores for some datasets \cite{zadrozny2002transforming,bennett2000assessing}. Although Platt's method is computationally efficient and prevents over-fitting, it is restrictive since the scores outputs are directly used regardless of their noise or errors \cite{jiang2012calibrating}.

In recent years, less restrictive non-parametric learning methods have gained popularity to address the issues of parametric learning approaches \cite{gokcesu2021nonparametric}. One such method is equal frequency histogram binning, also known as quantile binning or just binning \cite{zadrozny2001learning, zadrozny2001obtaining}. In this method, score outputs are sorted and partitioned into bins and each bin is optimized individually \cite{zadrozny2002transforming}. For small or unbalanced datasets, arbitrary bin numbers and boundaries can lead to inaccurate estimates. To address these, several extensions and refinements have been proposed, including ACP \cite{jiang2012calibrating}, which derives a confidence interval around each prediction to build the bins; and BBQ \cite{naeini2015obtaining}, which considers multiple binning models and their combination with a Bayesian scoring function \cite{heckerman1995learning}. ABB \cite{naeini2015binary} improves on these methods by considering Bayesian averaging over all binning models. However, none of these approaches utilizes the discrimination power of the input estimator.

The isotonic regression has become one of the most commonly used non-parametric fitting techniques in machine learning \cite{robertson1988order,gokcesu2021optimally,ayer1955empirical, brunk1972statistical,tibshirani2011nearly,gokcesu2021efficient,naeini2016binary,menon2012predicting,zhong2013accurate,gokcesu2022log}. It is an intermediary approach between parametric fitting and binning, which models the mapping based on the ranking of the score outputs \cite{zadrozny2002transforming}. Isotonic regression is a non-parametric regression that uses algorithms like pair-adjacent violators (PAV) \cite{ayer1955empirical, brunk1972statistical}. Various approaches that address the issues of binning have incorporated isotonic regression such as ENIR \cite{naeini2016binary} that utilizes modified pool adjacent violators algorithm (mPAVA) to find the solution path to a near isotonic regression problem in linearithmic time \cite{tibshirani2011nearly} and combines the predictions made by these models. Another variation of the isotonic-regression-based method is for predicting accurate probabilities with a ranking loss \cite{menon2012predicting}. Also, another extension combines the outputs from multiple binary classifiers to obtain calibrated probabilities \cite{zhong2013accurate}.
 
In this work, we show that the best unimodal isotonic fit for linear losses is a rectangular mapping. We provide a linearithmic time algorithm that finds the optimal mapping in a follow the leader manner. 
	
\section{Preliminaries}
\subsection{Problem Definition}\label{sec:problem}
We start by formally defining the problem setting of creating the optimal unimodal fit under linear losses. Let us have $N$ number of samples with their respective indices $n\in\{1,2,\ldots,N\}$. 

First, we have the input scores $\{x_n\}_{n=1}^N$ such that
\begin{align}
	x_n\in\overline\Re,&&\forall n.
\end{align}

Secondly, we have their respective linear losses $\{z_n\}_{n=1}^N$ such that
\begin{align}
	z_n\in[Z_l,Z_u],&&\forall n,
\end{align}
for some $Z_l,Z_u\in\Re$ and
\begin{align}
	Z_l\leq 0\leq Z_u.
\end{align}

Our goal is to optimally map the scores $\{x_n\}_{n=1}^N$ to $\{y_n\}_{n=1}^N$ such that
\begin{align}
	C(x_n)=y_n, &&\forall n,
\end{align}
where the mapping $C(\cdot)$ is unimodal, i.e., 
\begin{enumerate}

	\item We map these score values to values
	\begin{align}
	q_n\triangleq C(x_n)\in[Q_0,Q_1].\label{eq:q}
	\end{align}
	\item The mapping $C(\cdot)$ is unimodal, i.e.,
	\begin{align}
	q_n\geq& q_{n'} \text{ if } x_{n^*}\geq x_n\geq x_{n'},
	\\q_n\leq& q_{n'} \text{ if } x_{n^*}\leq x_n\leq x_{n'},\label{eq:CL}
	\end{align}
	for some $x_{n^*}$ 
\end{enumerate}
For this new setting, we have the following problem definition.
\begin{definition}\label{def:problem}
	For $\{x_n\}_{n=1}^N$, $\{z_n\}_{n=1}^N$, the minimization of the linear loss is given by
	\begin{align*}
	\argmin_{C(\cdot)\in\Omega_u}\sum_{n=1}^{N}z_nC(x_n),\label{eq:problemL}
	\end{align*}
	where $\Omega_u$ is the class of all univariate unimodal functions that map to the interval $[Q_0,Q_1]$.
\end{definition}
We point out that the problem in \autoref*{def:problem} fully generalizes many classification problems. Specifically,
\begin{itemize}
	\item For binary classification, we have
	\begin{align*}
		[Q_0,Q_1]=[0,1]
	\end{align*}
	\item If the classification is unweighted, we have
	\begin{align*}
		z_n=1-2y_n, &&[Z_0,Z_1]=[-1,1]
	\end{align*} 
	\item If there is a class weight, we have
	\begin{align*}
		z_n=1-(\alpha+1)y_n, &&[Z_0,Z_1]=[-\alpha,1]
	\end{align*}
	\item If there are sample weights, we have
	\begin{align*}
		z_n=\beta_n-\beta_n(\alpha+1)y_n, &&[Z_0,Z_1]=[-\alpha\beta,\beta],
	\end{align*}
	where $\beta\geq\beta_n$, $\forall n$.
\end{itemize}

\begin{assumption}\label{ass:dummyL}
	We have the following two sample pairs
	\begin{align*}
		z_1=Z_u,\text{     } x_1=-\infty && z_N=Z_u, \text{     }x_N=-\infty,
	\end{align*}
	to enforce the unimodality. Otherwise, we can arbitrarily add these dummy samples, which does not change the result of the original problem.
\end{assumption}
\subsection{Optimality of Thresholding for Linear Losses}\label{sec:opt}

Here, we show why the optimal unimodal transform $C(\cdot)$ on the scores $x_n$ is a rectangular function for the general linear losses.
Let us assume that there exists an optimal monotone transform $C^*(\cdot)$ in $\Omega_u$ that minimizes \autoref*{def:problem}, where $\Omega_u$ is the class of all unimodal functions that map to the interval $[Q_0,Q_1]$.
\begin{proposition}
	If $z_n\geq0$ for all $n$. The optimal transform will be $q_n=Q_0$ for all $n$.
	If not all $z_n$ are nonnegative; there exists some $n_*$ such that $z_{n^*}<0$ and the optimal transform will have $q_{n^*}=Q_1$, which is the peak of the unimodal transform.
\end{proposition}

\begin{lemma}\label{lem:sample}
If $C^*(\cdot)$ is an optimal transform with peak $n^*$ for \autoref*{def:problem}, then
\begin{align*}
C^*(x_n)=q^*_n=
\begin{cases}
q^*_{n+1}, & z_n<0, 1<n<n^*\\
q^*_{n-1}, & z_n>0, 1<n<n^*\\
q^*_{n+1}, & z_n>0, n^*<n<N\\
q^*_{n-1}, & z_n<0, n^*<n<N
\end{cases},
\end{align*}
for $q^*_1=Q_0$, $q^*_{n^*}=Q_1$, $q^*_N=Q_0$, which are the dummy samples from \autoref*{ass:dummyL} and the peak.
\end{lemma}
\begin{proof}
	The proof is straightforward since $q^*_n$ are unimodal with peak $n^*$ and changing $q^*_n$ accordingly decreases \autoref*{def:problem}.
\end{proof}
Hence, there are groups of samples with the same $q^*_n$. Let there be $I$ groups, where the group $i\in\{1,\ldots,I\}$ cover the samples $n\in\{n_i+1,\ldots,n_{i+1}\}$ ($n_1=0$ and $n_{I+1}=N$). Let the new groups have the peak $i^*$ with $n^*\in\{n_{i^*}+1,\ldots,n_{i^*+1}\}$.
\begin{lemma}\label{lem:group}
	If $C^*(\cdot)$ is an optimal classifier from \autoref*{def:problem} with peak sample $n^*$ and group $i^*$, then
	\begin{align}
	q^*_i=	
	\begin{cases}
	q^*_{i+1},&  \sum_{n=n_i+1}^{n_{i+1}} z_n<0, 1<i<i^*\\
	q^*_{i-1},& \sum_{n=n_i+1}^{n_{i+1}} z_n>0,1<i<i^*\\
	q^*_{i+1},&  \sum_{n=n_i+1}^{n_{i+1}} z_n>0,i^*<i<I\\
	q^*_{i-1},& \sum_{n=n_i+1}^{n_{i+1}} z_n<0,i^*<i<I
	\end{cases}
	\end{align}
	for $q^*_1=Q_0$, $q^*_I=Q_0$ and $q^*_{i^*}=Q_1$.
\end{lemma}
\begin{proof}
	The proof follows \autoref*{lem:sample} and its proof.
\end{proof}
	
After establishing that sample mappings group together, we reach the following result.

	\begin{theorem}\label{thm:thresholdL} 
		There exist an optimal classifier $C^*(\cdot)\in\Omega_u$ (where $\Omega_u$ is the class of all unimodal functions that map to $Q_0,Q_1$) that minimizes \autoref*{def:problem} such that
		\begin{align}
		C^*(\{x_n\}_{n=1}^{{\tau_1}})\triangleq q^*_1=Q_0,\\
		C^*(\{x_n\}_{n=\tau_1+1}^{{\tau_2}})\triangleq q^*_2=Q_1,\\
		C^*(\{x_n\}_{n=\tau_2+1}^{{N}})\triangleq q^*_3=Q_0,
		\end{align}
		for some $\tau_1,\tau_2\in\{1,\ldots,N-1\}$.
		\begin{proof}
			The proof follows from \autoref*{lem:group}. If the $i^{th}$ group's cumulative linear loss is not zero then $q^*_i$ is either $q^*_{i-1}$ or $q^*_{i+1}$. Otherwise if it is zero it does not matter. Hence, there exists three distinct groups with the respective mappings $Q_0,Q_1,Q_0$.
		\end{proof}
	\end{theorem}

\section{Offline Algorithms for the Optimal Rectangle}\label{sec:finding}
In this section, we propose algorithms that can find an optimal unimodal transform, or equivalently, an optimal rectangle for the problem in \autoref*{def:problem}. 

\subsection{Useful Definitions}
\begin{definition}\label{def:A}
	Let us have an ordered set of some samples $\X=\{x_1,x_2,\ldots,x_N\}$ and their corresponding linear losses $\Z=\{z_1,z_2,\ldots,z_N\}$.  which is represented by the set
	\begin{align*}
		\A=\{\X,\Z\}.
	\end{align*} 
\end{definition}
Given a set $\A$, its solution is summarized as the following.
\begin{definition}\label{def:B}
	For a given set $\A$ as in \autoref*{def:A}, let the optimal thresholds be $x_k,x_m$ for $\A$, such that the three adjacent sets $\X_0=\{x_1,\ldots,x_k\}$, $\X_1=\{x_{k+1},\ldots,x_m\}$ and $\X_2=\{x_{m+1},\ldots,x_N\}$ have the minimizer mappings $q_1*,q_2^*,q_3^*$, which are $Q_0,Q_1,Q_0$ respectively. We define the auxiliary set with
	\begin{align*}
	\B=\{X_{1},X_{2},L_0,L_1,L_2\},
	\end{align*} 
	where $X_{1}=x_k,X_{2}=x_{m}$ are the threshold samples and $L_0=\sum_{n=1}^{k}z_n,L_1=\sum_{n=k+1}^{m}z_n,L_2=\sum_{n=m+1}^{N}z_n$ are the corresponding cumulative linear losses. 
\end{definition}
The set $\B$ in \autoref*{def:B} completely captures the solution and its corresponding cumulative loss. Given the set $\B$, our monotone transform is $C(x\leq X_{1})=Q_0,C(X_{1}< x\leq X_{2})=Q_1,C(X_{2}< x)=Q_0$, with the resulting cumulative loss $L=Q_0L_0+Q_1L_1+Q_0L_2$.

\subsection{Brute Force Approach: Batch Optimization in $O(N^3)$}\label{sec:brute}
	For a given $\A$ as in \autoref*{def:A}, to find a $\B$ as in \autoref*{def:B}, the brute force approach is to try all possible $(k,m)$ pairs as $k\in\{1,\ldots,N-1\},m\in\{k+1,\ldots,N\}$, i.e.,
	\begin{align}
		\B_{k,m}=\left\{x_k,x_m,\sum_{n=1}^{k}z_n,\sum_{n=k+1}^{m}z_n,\sum_{n=m+1}^{N}z_n\right\},
	\end{align}
	whose cumulative losses are given by
	\begin{align}
		L_{k,m}=Q_0\sum_{n=1}^{k}z_n+Q_1\sum_{n=k+1}^{m}z_n+Q_0\sum_{n=m+1}^{N}z_n
	\end{align}
	Then, we can choose the optimal pair $(k^*,m^*)$ with the minimum loss, i.e.,
	\begin{align}
	\B=\B_{k^*,m^*},
	\end{align}
	where
	\begin{align}
		(k^*,m^*)\triangleq \argmin_{k,m} L_{k,m},
	\end{align}
	which takes $O(N^3)$ time since computing $L_{k,m}$ takes $O(N)$ time for every $(k,m)$ pair, which are $O(N^2)$ in number.
	
\subsection{Iterative Approach: Batch Optimization in $O(N^2)$}\label{sec:iterative}
We observe that $L_{k,m}$ is iteratively calculable as
\begin{align}
	L_{k,m+1}=&L_{k,m}+z_{m+1}(Q_1-Q_0),\\
	L_{k,m-1}=&L_{k,m}+z_{m}(Q_0-Q_1),\\
	L_{k+1,m}=&L_{k,m}+z_{k+1}(Q_0-Q_1),\\
	L_{k-1,m}=&L_{k,m}+z_{k}(Q_1-Q_0).
\end{align}
We can calculate each pair as the following:
\begin{itemize}
	\item Set $k=1$, $m=2$
	\item Calculate $L_{k,m}$ (in $O(N)$ time)
	\item WHILE $k+m<2N-1$
	\subitem IF $k$ is odd and $m<N$
	\subsubitem $m'\leftarrow m+1$
	\subitem ELSE IF $k$ is odd and $m=N$
	\subsubitem $k'\leftarrow k+1$
	\subitem ELSE IF $k$ is even and $m>k+1$
	\subsubitem $m'\leftarrow m-1$
	\subitem ELSE 
	\subsubitem $k'\leftarrow k+1$
	\subsubitem $m'\leftarrow m+1$
	\subitem Calculate $L_{k',m'}$ using $L_{k,m}$ (in $O(1)$ time)
	\subitem $k\leftarrow k'$
	\subitem $m\leftarrow m'$
\end{itemize}
The above procedure iteratively calculates all $(k,m)$ pairs by fixing $k$ and traversing over all possible $m$.
In total, this approach iteratively calculates all $L_{k,m}$ in $O(N^2)$. Minimization also takes $O(N^2)$.

\subsection{Linear Time Approach}\label{sec:offN}
We next propose a linear time algorithm. It works as follows
\begin{itemize}
	\item Let $L_{0,1}=z_1,L_{1,1}=0$, $k_1=1$.
	\item FOR $n\in\{2,\ldots,N\}$
	\subitem $L_{1,n}\leftarrow L_{1,n-1}+z_n$
	\subitem $L_{0,n}\leftarrow L_{0,n-1}$
	\subitem $k_n\leftarrow k_{n-1}$
	\subitem IF $L_{1,n}>0$
	\subsubitem $L_{0,n}\leftarrow L_{0,n}+L_{1,n}$
	\subsubitem $k_n\leftarrow n$
\end{itemize}
The procedure above calculates best $0-1$ step fit for the first $n$ samples for $n\in\{1,\ldots,N\}$; and takes $O(N)$ time. Similarly, we can run the same procedure for the last $n$ samples to find the best $1-0$ step fit. Let the respective losses be $M_{1,n},M_{0,n}$ and the threshold be $m_n$ for the last $n$ samples. We then calculate the following losses:
\begin{align}
	L_n=Q_0L_{0,n}+Q_1L_{1,n}+Q_1M_{1,N-n}+Q_0M_{0,N-n},
\end{align}
which takes $O(N)$ time. Let the minimizer be
\begin{align}
	n_*=\argmin_n L_n.
\end{align}
Then, the respective two optimal thresholds are given by the pair $(k_{n_*},m_{N-n^*}-1)$.

\section{Online Linearithmic Time Algorithm}\label{sec:recursive}
Although the method presented in \autoref{sec:offN} is very efficient for determining the thresholds in batch optimization, its online implementation is poor. When a new sample is observed, we have to run the batch algorithm again to find the best thresholds, which results in a complexity of $O(N^2)$. Here, we propose a recursive approach that can be implemented in an online manner, which has optimal $O(\log(N))$ complexity per sample since just ordering takes $O(\log N)$ time per sample.

To find the optimal thresholds $k^*,m^*$ (or the samples $X_{1},X_{2}$), we implement a recursive algorithm as follows:
\begin{enumerate}
	\item At the bottom level of the recursion, we have the sets $\A^{1,n}=\{\X^{1,n},\Z^{1,n}\}$ and the corresponding $\B^{1,n}$ for $n\in\{1,\ldots,N\}$, where $\X^{1,n}=\{x_n\}$, $\Z^{1,n}=\{z_n\}$.
	\item Starting from level $k=1$ (the initial stage), we create sets $\A^{k+1,n}$ and $\B^{k+1,n}$ at every level $k$ by merging adjacent sets. When we merge $\A^{k,i}$ with $\A^{k,i+1}$ (and their corresponding $\B$ sets), we create new sets $\A^{k+1,j}$ and $\B^{k+1,j}$. If there are no adjacent sets to merge with a given $(\A^{k,i}$, $\B^{k,i})$ pair, then they are moved up one level, meaning that $\A^{k+1,j}=\A^{k,i}$ and $\B^{k+1,j}=\B^{k,i}$. 
	It is important to note that $i$ and $j$ are distinct relative indices at levels $k$ and $k+1$, respectively.
	\item The corresponding $\B$ sets are updated accordingly to reflect the optimal mapping for the respective $\A$ sets. In addition to the auxiliary set $\B$; we also keep track of sets $\C$ and $\D$ for every $\A$ which reflect, respectively, the best nondecreasing and nonincreasing mappings as per \cite{gokcesu2021optimally}. Similar to $\B$ sets, $\C$ and $\D$ are given by
	\begin{align*}
		\C=&\{XC,LC_{-},LC_{+}\},&&\D=\{XD,LD_{+},LD_{-}\},
	\end{align*}
	where $XC$, $XD$ are the respective threshold values; $LC_{-},LD_{-}$ and $LC_{+},LD_{+}$ are the cumulative losses of the samples that are mapped to $Q_0$ and $Q_1$, respectively.
\end{enumerate} 

Given a sequence of samples $x_n$ and their respective linear losses $z_n$; our objective is to determine the optimal thresholds for the observed samples so far. With each new arriving sample, the sets in the recursion are updated as follows:
\begin{enumerate}
	\item Suppose the recursion is run over the past samples observed, which merges the sets at each level. After a new sample, just update the necessary intermediate sets. 
	\item If a new set $\A_m$ is not between a pair of already combined sets at an arbitrary level $k$, schedule $\A_m$ to be moved up as itself to the next level $k+1$.
	\item Whenever there is a new set (middle) $\A_m$ between an already combined pair of left $\A_l$ and right $\A_r$ sets at an arbitrary level $k$, combine $\A_l$ with $\A_m$ instead of $\A_r$  (left-justified bias), i.e., the old combination of $\A_l$ and $\A_r$ is updated with the new combination of $\A_l$ and $\A_m$. Schedule $\A_r$ to be moved up as itself to the level $k+1$.
	\item If a scheduled-to-be-moved-up set has an adjacent already-moved-up set at a level $k$, replace the already-moved-up set with their combination at $k+1$. Otherwise, move up the scheduled-to-be-moved-up set to $k+1$.
	\item Whenever an intermediate set is updated (a new or updated combination) at an arbitrary level $k$, update the subsequent combinations at $k+1$.
\end{enumerate} 

Next, we prove that our sequential algorithm has logarithmic in time complexity for each new sample.
\begin{lemma}\label{lem:depth}
	The recursion has a depth of $D$ where $k \leq D$ for all $k$ and this depth is of the order $O(\log N)$ where $N$ is the number of samples.
	\begin{proof}
		Due to the structure of our sequential implementation, no two adjacent sets move up the recursion as themselves. Therefore, if the number of sets at level $k$ is $N_k$, the number of sets at level $k+1$ is bounded by $N_{k+1} \leq \frac{2}{3}N_{k} + \frac{1}{3}$ (assuming a scenario where the sets alternate between a set that moves up and two sets that combine), which completes the proof.
	\end{proof}
\end{lemma}

\begin{lemma}\label{lem:traverse}
	We traverse the recursion twice to update the relevant intermediate sets with the arrival of a new sample.
	\begin{proof}
		Every update at level $k+1$ is related to an update at level $k$, including new combinations. With the arrival of each new sample, we have at most two updates at the bottom level and we at most traverse the recursion for each of these updates individually, which results in the lemma.
	\end{proof} 
\end{lemma}

\begin{lemma}\label{lem:update}
Suppose we have two sets $\A^{0}$ and $\A^{1}$ with mutually exclusive and adjacent $\X^{0}$ and $\X^{1}$, where $\X^{0}$ and $\X^{1}$ are individually ordered and the last element of $\X^{0}$ is less than or equal to the first element of $\X^{1}$. Let $\B^0=\{X_{1}^{0},X_{2}^{0},L_0^{0},L_1^{0},L_2^{0}\}$ and $\B^1=\{X_{1}^{1},X_{2}^{1},L_0^{1},L_1^{1},L_2^{1}\}$ be the auxiliary sets of $\A^{0}$ and $\A^{1}$, respectively. Let $\C^0=\{X^0,L_-^0,L_+^0\}$ and $\D^1=\{X^1,L_+^1,L_-^1\}$ be the respective sets that reflect the best nondecreasing mapping for $\A^0$ and the best nonincreasing mapping for $\A^1$. Let $\A=\{\X^0\cup\X^1, \Z^0\cup\Z^1\}$ be the merging of $\A^{0}$ and $\A^{1}$. Then, the auxiliary set of $\A$ is given by:
	\begin{align}
	\B=	
	\begin{cases}
	\{X_1^0,X_2^0,L_0^{0},L_1^{0},\gamma+L_1^1+L_2^1\},&  L_1^0\leq \min (L_1^1,\gamma)\\
	\{X_1^1,X_2^1,L_0^0+L_1^0+\gamma,L_1^1,L_2^{1}\},& L_1^1\leq \min(L_1^0,\gamma)\\
	\{X^0,X^1,L_-^0,\theta,L_-^{1}\},& \gamma\leq \min(L_1^0,L_1^1)
	\end{cases},\nonumber
	\end{align}
	where $\gamma=L_2^0+L_0^1$, $\theta=L_+^0+L_+^1$.
	\begin{proof}
	When combining the two auxiliary sets $\B^0$ and $\B^1$, there exist three possibilities: both of the thresholds are from $\X^0$, both of the thresholds are from $\X^1$, or the first threshold is from $\X^0$ and the second threshold is from $\X^1$. From these possibilities, the one with the smaller loss will be optimal. 
	\end{proof}
\end{lemma}

\begin{theorem}\label{thm:complexity}
	The algorithm for updating the optimal thresholds sequentially has $O(\log N)$ tie complexity for each sample.
	\begin{proof}
		The sequential update method described is able to update the optimal threshold with a time complexity of $O(\log N)$ per sample. This is because, according to lemmas \ref{lem:depth} and \ref{lem:traverse}, the total number of updates required for each new sample is $O(\log N)$. Additionally, according to lemma \ref{lem:update}, each update takes $O(1)$ time. The auxiliary sets $\C$ and $\D$ are also updated with $O(\log N)$ time per samples \cite{gokcesu2021optimally}. Therefore, the total time required to find the new threshold is $O(\log N)$.
	\end{proof}
\end{theorem}
This statement suggests that even the best possible algorithm for finding the order of a new sample's score output $x_n$ takes at least $O(\log N)$ time. Therefore, since the sequential update can find the new thresholds in $O(\log N)$ time, it is efficient.

\section{Conclusion}\label{sec:conclusion}
In summary, we addressed the problem of finding the optimal unimodal transform to minimize linear loss models. We demonstrated that the optimal transform is a rectangular function. We proposed a linearithmic time sequential recursion algorithm to update the optimal rectangular fit given a stream of data, which may have various use cases such as signal activation detection, target region estimation and classification.

\bibliographystyle{ieeetran}
\bibliography{double_bib}
\end{document}